\documentclass[conference]{IEEEtran}
\usepackage{amsfonts}
\usepackage{amsmath}
\usepackage{amssymb}
\usepackage{algorithm}
\usepackage{algorithmic}
\usepackage{cite}
\usepackage{hyperref}
\usepackage[all]{hypcap} 

\ifCLASSINFOpdf
  \usepackage[pdftex]{graphicx}
     \DeclareGraphicsExtensions{.pdf,.jpeg,.png}
\else
     \usepackage[dvips]{graphicx}
     \DeclareGraphicsExtensions{.eps}
\fi
\newtheorem{theorem}{Theorem}

\begin{document}
\title{An Intelligent System to Detect, Avoid and Maintain Potholes: A Graph Theoretic Approach}

\author{
 \IEEEauthorblockN{Shreyas Balakuntala\IEEEauthorrefmark{1}, Sandeep Venkatesh\IEEEauthorrefmark{2}}
 \IEEEauthorblockA{M S Ramaiah Institute of Technology, Bangalore, India 560-054\\}
\IEEEauthorblockA{\IEEEauthorrefmark{1}shreyasbs1@yahoo.com, \IEEEauthorrefmark{2}vsandeepnaidu@gmail.com
}\\

}
\maketitle

\begin{abstract}
 In this paper, we propose a conceptual framework where a centralized system detects and assists the driver to avoid potholes on roads. The system also identifies the potholes which are to be repaired immediately. The system we propose comprises a laser sensor and pressure sensors in shock absorbers to detect and quantify the intensity of a pothole, a centralized server which maintains a database of locations of all the potholes which can be accessed by another unit inside the vehicle. A point to point connection device is also installed in vehicles so that, when a vehicle detects a new pothole, all the vehicles within a range of 20 meters are warned about the pothole. We propose an algorithm which computes a route with least number of potholes which is nearest to the desired destination. If the destination is unknown, then the system will check for potholes in the current road stretch and displays the level of damage. Also, the destination can be added, removed or changed any time during the travel. We prove that the algorithm returns an optimal path.
\end{abstract} 

\section{Introduction}

Millions of dollars are spent in maintaining and repairing potholes by municipalities around the world\footnote{\url{http://www.bizjournals.com/denver/stories/2007/04/02/story1.html?page=all}}. A pothole refers to a shallow pit on a road's surface, caused by activities like erosion, weather, traffic and some other factors. These anomalies when accumulated in the transportation system, constitutes to major problems. These problems, even though they appear to be less significant at an individual level, constitute to major problems when taken in cumulative, collective and large scale manner. The problems constituted by these potholes result in low fuel economy, accidents, traffic coagulations etc, which have an adverse impact on the economy of a country and day to day life of citizens. It can be proved that the traffic system can be bettered to a greater extent when these problems are checked and taken care of. Detecting Potholes on roads with the help of intelligent systems is a very well-studied problem. Detecting and hence avoiding potholes may reduce the fuel consumption, wear-tear and maintenance cost of a vehicle. Also, avoiding potholes increases road safety and indirectly decreases the total travel time in some cases.The existing systems of pothole detection uses a centralized database approach.

Road conditions are a matter of public concern that have engendered a number of responses from local organizations dissatisfied with the state of their roads. Some of the solutions include: establishing “pothole hot-lines”, holding contests to report particularly bad potholes, and asking readers to contribute pictures of potholes. We seek a more systematic approach to the problem, but hope that this public interest may cause volunteers to carry hardware in their cars.
 
The existing systems propose a way to detect potholes. But, the vehicles travelling on these damaged roads eventually take damages. We try to reduce this problem by proposing a novel conceptual framework where the detection, avoidance and maintenance of potholes are taken care of to a great possible extent. Our system accounts for a better,faster and reliable method in detection, avoidance and maintenance of potholes. Thus this architecture, when employed brings about a radical change in existing roadways transportation system.

\subsection{Related Work}
There are several vehicular sensing methods for pothole detection. Some of these use accelerometers or lasers for data acquisition. This section contains a short review of pothole detection and avoidance algorithms implemented. Some techniques use image processing for the detection.\\
 
Mertz et. al., in \cite{a} uses service vehicles to detect road damages. The system consists of a structured light sensor and a camera mounted on vehicles that travel the roads on a regular basis. It makes use of sensors and equipment already present on the vehicle, like GPS on transit buses. The data is collected from many vehicles, aggregated and analysed at a central location and the assessment results are displayed interactively to facilitate road maintenance operations.
Rode et. al., \cite{b} proposed a pothole detection and warning system which is divided into three subsystems. First is sensing subsystem which senses the potholes encountered by it, about which it did not have the prior information. Second is communication subsystem which handles the information transfer between Wi-Fi Access Point and Mobile Node. Third subsystem is the localization subsystem which analyzes the data received from Access Points and warns the driver regarding the occurrence of potholes.
De Zoysa et al., in \cite{d} proposes a method of detecting potholes where pothole data and its GPS coordinates are stored in the sensor. This data stored is  uploaded to the base station  when the bus reaches the station. The station acts as an central hub and in turn is handed over to the buses that travel in that route.\\
An image processing approach was proposed by Ajit et.al., in \cite{c} where the potholes are first photographed or recorded using camera on a car. This image processes the road such that the pothole appears as a distinct black colour. The images can be filtered to remove unnecessary objects like other cars using proper pattern matching. This efficiently detects the presence of potholes on the roads. Jengo et. al., in\cite{j} uses hyperspectral imagery to characterise the quality of roads, the approach is of semi-automated in nature.\\

Smartphones have been used for effective detection and avoiding mechanisms. An algorithmic approach was proposed by Eriksson et. al.,\cite{e} in addition to static road sensing.It describes the design, implementation, and experimental evaluation of Pothole Patrol, a mobile sensor computing system to monitor and assess road surface conditions.\\
An approach where accelerometers in Android Smartphones are used to detect potholes \cite{f}. The system describes a mobile sensing system  for road irregularity detection using Android OS based smart-phones. It maps the road system by using smartphones and does statistical analysis of the mapped data using existing RoadMic Technology. Strazdins et. al., in \cite{g} test the feasibility of such Android Smartphones for pothole detection.\\

The problem is well studied for autonomous robots too. A detection-avoidance mechanism was introduced for the navigational aid of autonomous vehicles\cite{h}. In this paper, they discuss a solution to detection and avoidance of simulated potholes in the path of an autonomous vehicle operating in an unstructured environment.
An obstacle avoidance system was developed for a custom-made autonomous navigational robotic vehicle (ANROV), based on an intelligent sensor network and fuzzy logic control\cite{i}.

\subsection{Definitions}
\subsubsection{Street Network} A street network is a directed multi-graph $G_d=(V,E_d)$ where $V$ is the set of nodes which correspond to interconnecting points and $E_d$ is the set of edges(arcs) which correspond to interconnecting lines that represent a system of streets for a given area. Here $|V|=N$ where N is the number of nodes. The graph is directed to account for multiple lanes and two-way streets between any 2 vertices. 

\subsubsection{ Weight Multiset and Min-Weight Multiset of a Directed Multigraph}
We borrow the idea of Weight Multiset and Min-Weight Multiset from Siddhartha Sankar Biswas et al.\cite{dij}.\\
Suppose that there are $n$ number of arcs from the vertex $u$ to the vertex $v$ in a directed multigraph $G_d$, where $n$ is a non-negative integer. $W_{uv}$ is the set whose elements are the arcs between vertex $u$ and vertex $v$, keyed and sorted in non-descending order by the value of their respective weights. $$\therefore W_{uv} = \{(uv_1, w_{1uv}), (uv_2, w_{2uv}),(uv_3, w_{3uv}),..$$
$$.., (uv_n, w_{nuv})\}$$, where, $uv_i$ is the arc-i from vertex $u$ to vertex $v$ and $w_{iuv}$ is weight of that arc, for i =\{1, 2, 3,.., n\}. If two or more number of weights is equal then ties are broken arbitrarily. 
\\The multiset $\{w_{1uv}, w_{2uv}, w_{3uv}, ... ,w_{nuv}\}$ is also denoted by the name $W_{uv}$. This $W_{uv}$ is called the 'Weight-Multiset'.Here $w_{uv}$ is the min value of the multiset $W_{uv} =\{w_{1uv}, w_{2uv}, w_{3uv}, ... ,w_{nuv}\}$. Clearly, $w_{uv} = w_{1uv}$, as the multiset $W_{uv}$ is already sorted. The collection of all $w_{uv}$ in a directed multigraph forms a multiset which is denoted by $W$ and is called by ‘Min-Weight Multiset’ of the Multigraph.
\\
\\
For simplicity, we denote $w_{iuv}$ as $w(e)$, where $e$ is the $i$-th arc between $u$ and $v$.

\subsubsection{Pothole Identification Number}
Each pothole will be assigned a unique pothole identification number $Id$ which will be linked to its location i.e, the arc on which it is present. Any transaction/ process for a pothole will be done through this number.

\subsection{Brief overview of the results}
We propose a method where the detection, avoidance and maintenance of potholes are taken care of to a great possible extent. We use a peer to peer communication, along with the centralized database approach. The vehicle will be routed such that the potholes in the path will be minimum. The potholes are regularly reported to the municipal corporation based on the number of hits. Hence the system accounts for a better,faster and reliable method in detection, avoidance and maintenance of potholes.

\section{Our Approach}
\begin{figure}[!t]
\centering
\includegraphics[width=9cm, height=5.5cm]{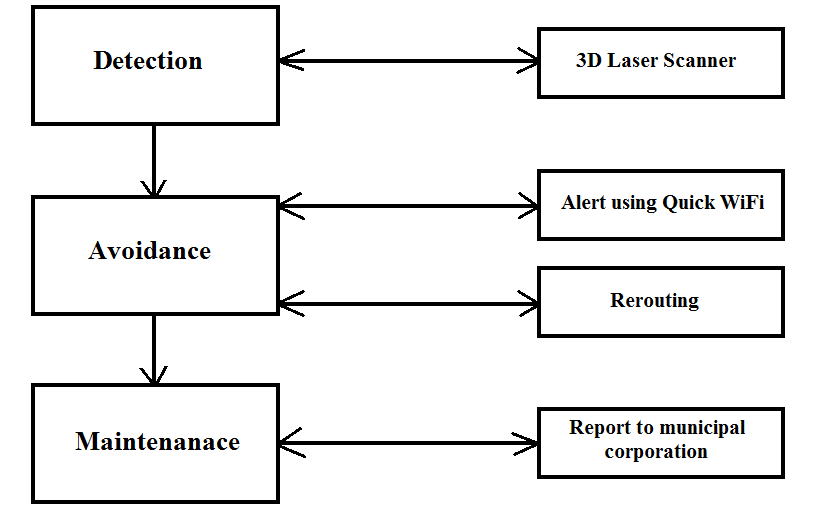}\label{1}
\caption{The proposed system with 3 components.}
\end{figure}
The proposed system is a conceptual framework which can be implemented for vehicles in metropolitan cities. The system consists of three components namely Detection, Avoidance and Maintenance of the potholes. Figure \ref{1} is a block diagram of the complete system. The components are implemented using a set of agent(s). Each component will be explained in detail in the following sections.  

\subsection{Detection}
\subsubsection{Pothole \textit{Detection} With Lasers}
We attach a laser detector and transmitting device to the vehicles' under carriage. The device utilises a laser to detect the depth of the pothole. A 3-D laser scanner \cite{a} operates by sweeping a laser across the scene in two dimensions. At each pixel, the instrument measures the time that it takes for a laser beam to leave the sensor, strike a surface, and return. The sensors also provide an intensity measurement at each pixel by measuring the energy of the returned laser signal. Thus, a full sweep in two-dimensions can provide both a depth map and an intensity image using Laser Detection and Ranging mechanisms \cite{yy}. The depth and location of a pothole are linked to a pothole identification $Id$.
 \subsubsection{Encryption and Decryption}
A data encryption and decryption system securely geoencrypts data using location-dependant navigation signals. To prevent any brute-force attack against the cryptographic key, we use largely time-independent characteristics of the navigation. It includes a geoencryption apparatus\cite{x} that includes the data file, consisting of the depth map and the intensity image, the location of the device along with a random key in order to create a coded file. The decryption decrypts the same using RSA algorithm\cite{rsa}. It discards the random number and utilises the location, depth map and intensity image for further processing.

\subsection{Avoidance}
\subsubsection{Warning: Local Transmission of Data}
To transfer data locally between the vehicles we use Cabernet\cite{z}. Cabernet is a system for delivering data to and from moving vehicles using open 802.11 (WiFi) access points(AP) encountered opportunistically during travel.Network connectivity in Cabernet is both fleeting (where access points are typically within range for a few seconds)and intermittent. The advantages of using Quick WiFi is that data transfers can occur in broadband speeds.

As a car drives down the road, the on-board embedded device repeatedly scans for, and attempts to associate with, open APs. It then attempts to establish end-to-end connectivity with a Cabernet enabled host, to retrieve or upload the data obtained. When a vehicle moves, the connections are typically brief. The device associates with an AP and maintains the connection until the access point went out of range. In the connection establishment process each step involves a request, followed by a response. If no response is received within the specified time, the request is retransmitted. QuickWiFi attempts to associate with the first open access point it encounters as it scans through the wireless channels.Open APs require the connection to be authenticated before the transmission can occur. In case the authentication is lost, the process is restarted.Upon success, QuickWiFi explicitly notifies applications that connection is available and the data is transmitted.We quickly discover that a connection is lost, if we have not seen any transmissions for 500 milliseconds, it is likely that the car has moved out of the range of the AP. When this happens, we are rescanning again to transmit data to other vehicles within the range.

\subsubsection{Transmission of Data to the Central Server}
Each device is a mini-server that supports a set of functions and can process portions of the queries\cite{y}. It is also connected to a central database which manages the pothole location discovered along with its depth. The database will store a tuple (Depth, Location, Intensity) of potholes detected and a pothole $Id$ will be pointing this tuple. On querying, the database return results immediately. They monitor the roads as a continuous phenomenon. The central database system accepts a query regarding the condition of the road. It produces a query execution plan for the query and executes this plan against the database containing the information about the location and depths of potholes.The execution plan is the internal blueprint for evaluating a query. The answer to the query is encrypted and then send to the requesting devices.Query processing takes advantage of the computing capabilities at the devices in order to  minimise reaction time. Figure \ref{2} shows the instance where all the vehicles are connected to each other and to the server through Cabernet.
\begin{figure}[!t]
\centering
\includegraphics[width=9cm, height=5.5cm]{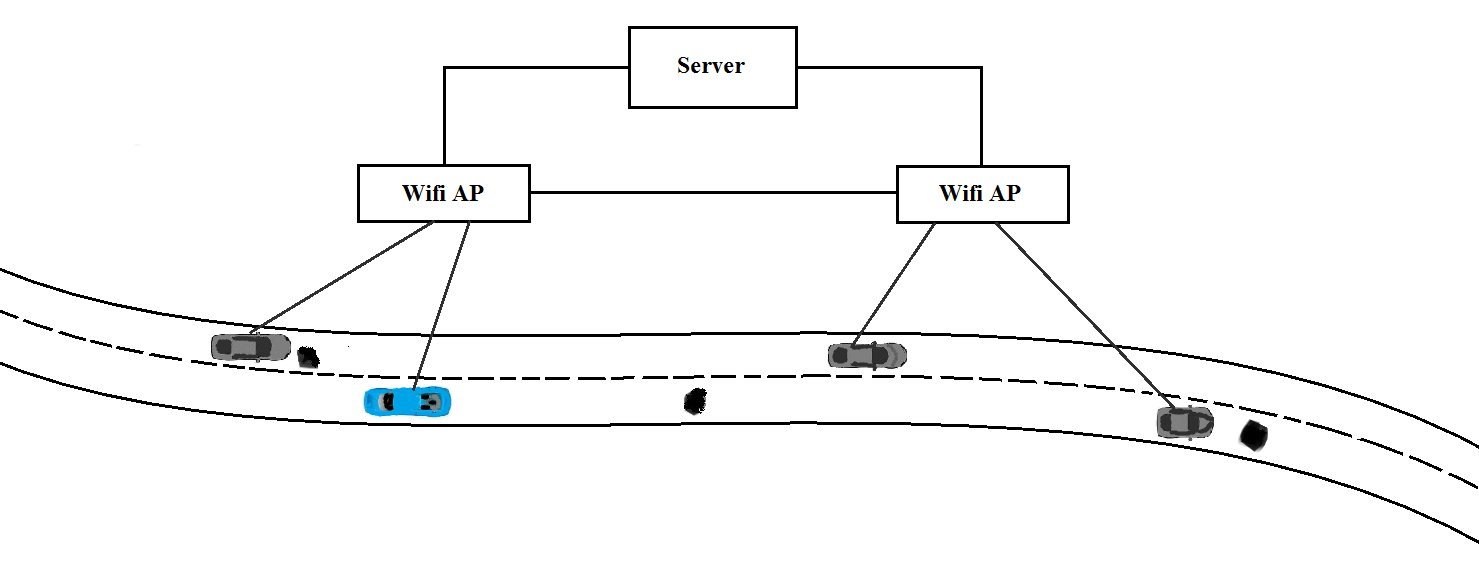}\label{2}
\caption{An instance where vehicles are connected to the server through a wifi access point. Here a new pothole is detected and is updated to the server by the rightmost and the leftmost vehicles.}
\end{figure}
\subsubsection{Preprocessing of the Street Network}
The street network will be preprocessed and weights will be added to all the arcs. These weights are based on the depth of the potholes( damage level of the arc) and the length of the arc which s the actual physical distance between the 2 vertices considered. The weights will be updated for every update of a pothole in the server. Algorithm \ref{alg:1} is used for the preprocessing. Here the weight $w_{iuv}$, of a given $i$-th arc $e$ between any 2 vertices $(u,v)$, is given by the product of the average damage of the arc and its length i.e., $$ w_{iuv}=d \times l $$. The average damage of an arc is given by $$d =\frac{\sum_{j=1}^{a} b_j} {a} $$ where $a$ is the number of potholes on the arc and $b_j$ is the depth of $j$-th pothole. We consider the product of $d$ and $l$ to account for both the damage and the distance. For an optimal path, both damage and the distance are to be minimum and hence the weight $w_{iuv}$ accounts for a better measure.    
\begin{algorithm}
\caption{Preprocessing of the Street Network}
\label{alg:1}
\begin{algorithmic}[1]
\REQUIRE Street Network $G_d=(V,E_d)$, Weight Multiset $W_{uv}$ for all $(u,v) \in V$  and Pothole Identification $Id$ for all the potholes in the Database as input.
\FOR{all $e \in E_d$}
\FOR{all $Id$ in the database}
\IF{Pothole with identification number $Id$ is on arc $e$} 
\STATE Arc-damage sum, $d'(e)= d'(e)+ b(Id)$ where $b(Id)$ is the depth of the pothole on arc $e$ with identification $Id$.
\ENDIF
\ENDFOR
\ENDFOR
\FOR{all $e$ between any $(u,v)$}
\STATE Average arc-damage, $d(e)= \frac{d'(e)}{a}$ where a is the number of potholes on arc $e$.
\STATE Weight of arc, $w(e)=d(e) \times l(e)$ where $l(e)$ is the length of the arc $e$ obtained from the weight multiset ($w(e)=w_{iuv}$ assuming $e$ to be the $i$-th arc between $u$ and $v$).
\ENDFOR
\RETURN Weighted Multi-graph $G_d'=(V,E_d)$ of Street network $G_d$.\\
\end{algorithmic}
\end{algorithm}

\subsubsection{Routing the Vehicle through Optimal Path}
The device in the vehicle will have a display unit, which shows the street network as a map to the driver. The system is flexible enough that the destination can be set at the source, while driving and in addition to it, the destination can be modified. If the destination is set at the source, a path which has least weight will be displayed. If its not set, then the weight of the arc on which the vehicle is travelling will be displayed. If the destination is modified, i.e., changed or deleted, one of the above 2 processes will be called. Algorithm2 will be followed for routing. In Algorithm \ref{alg:2}, we use Generalised Dijkstra's algorithm(GDA) for multi-graphs \cite{dij} to route the vehicle. This GDA algorithm returns shortest path from a source to all other vertices.
\begin{algorithm}
\caption{Routing to Destination}
\label{alg:2}
\begin{algorithmic}[1]
\REQUIRE  Weighted Multi-graph $G_d'=(V,E_d)$ , Min-Weight Multiset $W$, Source $S$ and Destination $D$ as input.\\
\IF{Destination $D=\emptyset$}
\STATE Get the GPS data of the vehicle i.e, Arc $e$ on which the vehicle is travelling.
\RETURN The weight $w(e)$. 
\STATE Repeat steps 2 and 3 until the destination is either Set or Reached for every node crossed.
\ELSE
\STATE Run GDA($G_d'$,$W$,$S$).
\RETURN The path between $S$ and $D$ obtained from GDA algorithm along with its weight.
\IF{$D$ is modified}
\STATE Repeat step 6 and 7.
\ENDIF
\ENDIF
\end{algorithmic}
\end{algorithm}

\begin{theorem}
Algorithm \ref{alg:2} always returns an optimal path.
\end{theorem} 
\begin{proof}
The weights to the multi-graph $G_d'=(V,E_d)$ is added such that both the length of an arc $l(e)$ and its average depth $d(e)$ are minimal. This results in the weight of an arc $w(e)= l(e) \times d(e)$ to be minimum. Generalized Dijkstra’s algorithm returns a path which is the shortest between given source-destination pair $(S,D)$. The path is said to be optimal since it bears a weight which accounts for less damage and a shorter distance. Hence Algorithm \ref{alg:2} returns an optimal path always. 
\end{proof}

\subsection{Maintenance: Determining Traffic Intensity using server updates for a particular pothole}
The location of every pothole is updated by every vehicle which communicates the central the database. We utilise this redundant effectively to find the traffic intensity of the roads. Traffic Intensity is equal to the number of traffic updates for a particular location per minute. The traffic intensity on a particular location is compared by the number of pothole updates received for a particular location in one minute. This gives an approximate idea of the traffic intensity in the damaged roads. This data is properly exploited to prioritise the potholes on the level of threat to traffic intensive roads.  The classified roads can be intimated to the corresponding authorities for action to be taken upon. 
 
Figure \ref{3} shows the complete system where a tuple (Depth, Location, Intensity) will be returned by the system upon detection of a pothole to the server. This tuple for a pothole is used by other vehicle for avoidance and maintenance purpose. 
\begin{figure}[!t]
\raggedleft
\includegraphics[width=8.5cm, height=5.5cm]{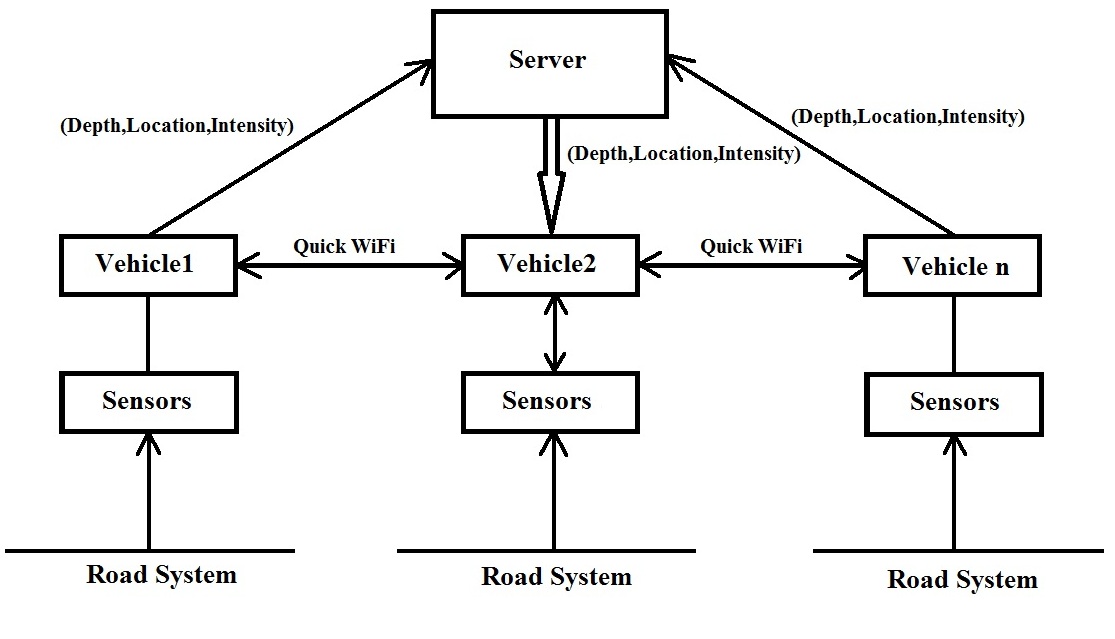}\label{3}
\caption{The proposed system as a whole.}
\end{figure}
\section{Conclusion}
This paper describes laser-based detection and avoidance of potholes using routing algorithms and road damage level predictions. It also prioritises the potholes in important roads based on traffic intensity. The method we use exploits the current location of the device and sends it to other vehicles using Quick Wifi APs. The same data is also forwarded to the central server. Thereby increasing the accuracy of the exact location of the pothole. The road damage level predicting algorithm gives a scale of damage for the road ahead. The routing algorithm returns an optimal path to the destination with minimal damage and a shorter distance.

\section{Further Work}
The data in the central server can be exploited by concerned authorities for remedial action on the potholes. The verification of the remedial action can also be taken by detecting the presence of potholes in the particular location. Also the efficiency of interaction of vehicles through Quick Wifi can be improved to reduce data losses. The routing for avoidance can be studied in detail and efficient algorithms can be proposed. Moreover, the system is still a theoretical proposition which can be implemented in metropolitan cities. 

\section{Acknowledgement}
The authors would like to thank M.S.Ramaiah Institute of Technology, Prof.K.G.Srinivasa( Head, Dept of CSE, M.S.Ramaiah Institute of Technology) the support and motivation. The authors would also like to thank Dr. Ranjini Tolakanahalli( Medical Physicist, Hamilton Health Sciences) and Mr. Balakunatala Prasanna( Chair, IEEE Region 1) for the review and suggestions.

\end{document}